\documentclass[letterpaper]{article}
\pdfoutput=1
\usepackage{aaai}
\usepackage{times}
\usepackage{helvet}
\usepackage{courier}
\usepackage{epic}
\usepackage{graphicx}
\usepackage{latexsym}
\usepackage{verbatim}
\usepackage{xspace}
\usepackage{amsmath}
\usepackage{amsthm}
\usepackage{amssymb}
\usepackage{multirow}
\begin{document}

\newlength{\halftextwidth}
\setlength{\halftextwidth}{0.47\textwidth}
\def\halffigsize{2.2in}
\def\thirdfigsize{1.5in}
\def\negvspace{0in}
\def\posvspace{0em}






\newcommand{\myset}[1]{\ensuremath{\mathcal #1}}

\renewcommand{\theenumii}{\alph{enumii}}
\renewcommand{\theenumiii}{\roman{enumiii}}
\newcommand{\figref}[1]{Figure \ref{#1}}
\newcommand{\tref}[1]{Table \ref{#1}}
\newcommand{\myOmit}[1]{}
\newcommand{\myldots}{.}

\newcommand{\nina}[1]{\marginpar{\sc nina} \textit{{#1}}}

\newtheorem{mydefinition}{Definition}
\newtheorem{mytheorem}{Theorem}
\newtheorem*{myexample}{Running example}
\newtheorem{mytheorem1}{Theorem}
\newtheorem{definition}{Definition}
\newtheorem{theorem}{Theorem}
\newtheorem{lemma}{Lemma}
\newtheorem{corollary}{Corollary}
\newcommand{\myproof}{\noindent {\bf Proof:\ \ }}
\newcommand{\myqed}{\mbox{$\Box$}}

\newcommand{\oalldiffshort}{\mbox{\sc oBC}}
\newcommand{\shc}{\mbox{\sc sim-HC}\xspace}
\newcommand{\shs}{\mbox{\sc sim-HallSet}\xspace}
\newcommand{\ashs}{\mbox{\sc a-sim-HallSet}\xspace}
\newcommand{\sbm}{\mbox{\sc sim-BM}\xspace}
\newcommand{\sset}{\mbox{\sc loose-Set}\xspace}
\newcommand{\mymod}{\mbox{\rm mod}}
\newcommand{\range}{\mbox{\sc Range}}
\newcommand{\roots}{\mbox{\sc Roots}}
\newcommand{\myiff}{\mbox{\rm iff}}
\newcommand{\alldifferent}{\mbox{\sc AllDifferent}\xspace}
\newcommand{\alldiff}{\mbox{\sc AllDifferent}\xspace}
\newcommand{\interdistance}{\mbox{\sc InterDistance}}
\newcommand{\permutation}{\mbox{\sc Permutation}}
\newcommand{\disjoint}{\mbox{\sc Disjoint}}
\newcommand{\cardpath}{\mbox{\sc CardPath}}
\newcommand{\CARDPATH}{\mbox{\sc CardPath}}
\newcommand{\knapsack}{\mbox{\sc Knapsack}}
\newcommand{\common}{\mbox{\sc Common}}
\newcommand{\uses}{\mbox{\sc Uses}}
\newcommand{\lex}{\mbox{\sc Lex}}
\newcommand{\LEX}{\mbox{\sc Lex}}
\newcommand{\SnakeLex}{\mbox{\sc SnakeLex}}
\newcommand{\usedby}{\mbox{\sc UsedBy}}
\newcommand{\atmostnvalue}{\mbox{\sc AtMostNValue}}
\newcommand{\atleastnvalue}{\mbox{\sc AtLeastNValue}}
\newcommand{\slide}{\mbox{\sc Slide}}
\newcommand{\SLIDE}{\mbox{\sc Slide}}
\newcommand{\circularslide}{\mbox{\sc Slide}_{\rm O}}
\newcommand{\among}{\mbox{\sc Among}}
\newcommand{\mysum}{\mbox{\sc Sum}}
\newcommand{\amongseq}{\mbox{\sc AmongSeq}}
\newcommand{\atmost}{\mbox{\sc AtMost}}
\newcommand{\atleast}{\mbox{\sc AtLeast}}
\newcommand{\element}{\mbox{\sc Element}}
\newcommand{\gcc}{\mbox{\sc Gcc}}
\newcommand{\nvalue}{\mbox{\sc NValue}}
\newcommand{\egcc}{\mbox{\sc EGcc}}
\newcommand{\gsc}{\mbox{\sc Gsc}}
\newcommand{\sequence}{\mbox{\sc Sequence}}
\newcommand{\contiguity}{\mbox{\sc Contiguity}}
\newcommand{\PRECEDENCE}{\mbox{\sc Precedence}}
\newcommand{\precedence}{\mbox{\sc Precedence}}
\newcommand{\assignnvalues}{\mbox{\sc Assign\&NValues}}
\newcommand{\linksettobooleans}{\mbox{\sc LinkSet2Booleans}}
\newcommand{\domain}{\mbox{\sc Domain}}
\newcommand{\symalldiff}{\mbox{\sc SymAllDiff}}
\newcommand{\valsymbreak}{\mbox{\sc ValSymBreak}}
\newcommand{\RowColSym}{\mbox{\sc RowColLexLeader}}
\newcommand{\RowColSymShort}{\mbox{\sc RowColLL}}
\newcommand{\RowSymShort}{\mbox{\sc RowLL}}
\newcommand{\ColSymShort}{\mbox{\sc ColLL}}
\newcommand{\NoSymShort}{\mbox{\sc NoSB}}
\newcommand{\RowLexLeader}{\mbox{\sc RowLexLeader}}
\newcommand{\OrderRowCol}{\mbox{\sc Order1stRowCol}}

\newcommand{\slidingsum}{\mbox{\sc SlidingSum}}
\newcommand{\MaxIndex}{\mbox{\sc MaxIndex}}
\newcommand{\REGULAR}{\mbox{\sc Regular}}
\newcommand{\regular}{\mbox{\sc Regular}}
\newcommand{\Regular}{\mbox{\sc Regular}}
\newcommand{\STRETCH}{\mbox{\sc Stretch}}
\newcommand{\SLIDEOR}{\mbox{\sc SlideOr}}
\newcommand{\NAE}{\mbox{\sc NotAllEqual}}
\newcommand{\mymax}{\mbox{\rm max}}

\newcommand{\calC}{\ensuremath{\mathcal{C}}}
\newcommand{\calL}{\ensuremath{\mathcal{L}}}
\newcommand{\calX}{\ensuremath{\mathcal{X}}}

\newcommand{\Xub}{\ensuremath{X^+}}
\newcommand{\Xlb}{\ensuremath{X^-}}
\newcommand{\Yub}{\ensuremath{Y^+}}
\newcommand{\Ylb}{\ensuremath{Y^-}}

\newcommand{\constraint}[1]{\mbox{\sc #1}\xspace}
\newcommand{\oalldiff}{\constraint{OverlappingAllDiff}}

\newcommand{\tuple}[1]{\left\langle #1 \right\rangle}

\newcommand{\todo}[1]{{\tt (... #1 ...)}}
\newcommand{\set}[1]{\left\{ #1 \right\}}

{\makeatletter
 \gdef\xxxmark{%
   \expandafter\ifx\csname @mpargs\endcsname\relax 
     \expandafter\ifx\csname @captype\endcsname\relax 
       \marginpar{xxx}
     \else
       xxx 
     \fi
   \else
     xxx 
   \fi}
 \gdef\xxx{\@ifnextchar[\xxx@lab\xxx@nolab}
 \long\gdef\xxx@lab[#1]#2{{\bf [\xxxmark #2 ---{\sc #1}]}}
 \long\gdef\xxx@nolab#1{{\bf [\xxxmark #1]}}
}

\newcommand{\DC}{\ensuremath{DC}\xspace}
\newcommand{\Xbf}{\mbox{{\bf X}}\xspace}
\newcommand{\LEXCHAIN}{\mbox{\sc LexChain}}
\newcommand{\DLex}{\mbox{\sc DoubleLex}}
\newcommand{\snakelex}{\mbox{\sc SnakeLex}}
\newcommand{\DLexColSum}{\mbox{\sc DoubleLexColSum}}

\title{Propagating Conjunctions of \alldiff Constraints}
\author{
Christian Bessiere\\ LIRMM, CNRS\\
Montpellier, France \\
bessiere@lirmm.fr\\
\And
George Katsirelos\\CRIL-CNRS\\
Lens, France\\
gkatsi@gmail.com
\And
Nina Narodytska\\
NICTA and UNSW\\
Sydney, Australia\\
ninan@cse.unsw.edu.au
\And
Claude-Guy Quimper\\ 
Universit\'{e} Laval\\
Qu\'{e}bec, Canada\\
cquimper@gmail.com\\
\And
Toby Walsh\\
NICTA and UNSW\\
Sydney, Australia\\
toby.walsh@nicta.com.au}

\maketitle
\begin{abstract}
We study propagation algorithms for the conjunction of two 
$\alldiff$ constraints. Solutions of an $\alldiff$ constraint 
can be seen as perfect matchings on the variable/value bipartite graph. 
Therefore, we investigate the problem of finding 
simultaneous bipartite matchings.
We present an extension of the famous Hall theorem 
which characterizes when simultaneous bipartite
matchings exists. Unfortunately, finding
such matchings is NP-hard
in general. However, we prove a surprising result that finding a
simultaneous matching on a convex bipartite graph takes 
just polynomial time. Based on this theoretical result,
we provide the first polynomial time bound consistency
algorithm for the conjunction of two $\alldiff$ constraints.
We identify a pathological
problem on which this propagator is exponentially
faster compared to existing propagators.
Our experiments show that this new propagator
can offer significant benefits 
over existing methods. 
\end{abstract}


\section{Introduction}

Global constraints 
are a critical factor in the success of
constraint programming. They capture patterns that often
occur in practice (e.g. ``these courses must occur
at different times''). 
In addition, fast propagation algorithms
are associated with each global constraint
to reason about potential solutions (e.g.
``these 4 courses have only 3 time slots between
them so, by a pigeonhole argument, 
the problem is infeasible''). 
One of the oldest and most useful global constraints is the \alldifferent
constraint~\cite{alice}. This specifies that
a set of variables takes all different values.
Many different algorithms have been proposed for 
propagating the
\alldifferent constraint
\cite{regin1,leconte,puget98
}. 
Such propagators can have
a significant impact on our ability
to solve problems 
\cite{swijcai99}
. 

Problems often
contain multiple \alldifferent constraints
(e.g. ``The CS courses must occur
at different times, as must the IT
courses. In addition, CS and IT
have several courses in common''). 
Currently, constraint solvers ignore
information about the overlap between multiple
constraints (except for the limited communication
provided by the domains of common
variables). Here, we show the benefits
of reasoning about such overlap. 
This is a challenging problem as 
finding a solution to just two \alldifferent constraints
is NP-hard \cite{KEKM08}
and  existing approaches to deal with such overlaps require
exponential space \cite{lardeux}. 
Our approach is to focus on domains
that are ordered, as often occurs in practice. For example
in our time-tabling problem, values might represent times (which are naturally
ordered). In such cases, domains can be compactly
represented by intervals. Propagation algorithms 
can narrow such intervals using the notion of bound consistency.
Our main result is to prove we can enforce
bound consistency on two \alldifferent constraints
in polynomial time. Our algorithm exploits
a 
connection with matching on bipartite graphs. 
In particular, we consider \emph{simultaneous} matchings.
By generalizing Hall's theorem, we 
identify a necessary and sufficient condition for the existence of
such a matching and show  that the 
this problem is polynomial for convex 
graphs. 
\myOmit{
. This condition shows the surprising 
result that the 
simultaneous matching problem is polynomial for convex bipartite
graphs. 
Based on these insights, we develop 
the first polynomial
time bound consistency algorithm for the conjunction of two $\alldiff$
constraints.
}
\section{Formal background}
\textbf{Constraint programming.}
We use capitals for variables 
and lower case for values. 
Values range over 1 to $d$. 
We write $D(X)$ for the domain of values for $X$, $lb(X)$ ($ub(X)$)
for the smallest (greatest) value in $D(X)$. 
A \emph{global constraint} is one in which the number of variables $n$
is a parameter. For instance,
$\alldiff([X_1,\ldots,X_n])$ 
ensures that 
$X_i \neq X_j$ for any $i<j$. 
%
Constraint solvers prune search
by enforcing properties like 
domain consistency. A constraint is 
\emph{domain consistent} (\emph{DC})
iff when a variable is assigned any value in its domain, there
are compatible values in the domains of all other variables.
Such an assignment is 
a \emph{support}. 
A constraint is \emph{bound consistent} (\emph{BC})
iff when a variable is
assigned the minimum or maximum value in its domain, there are compatible
values between the minimum and maximum domain value
for all other variables.
Such an assignment is 
a \emph{bound support}. 
A constraint is \emph{bound disentailed}
iff no possible assignment is a  bound support. 

\noindent\textbf{Graph Theory.} 
Solutions
of $\alldiff$ correspond to matchings
in a bipartite variable/value graph~\cite{regin1}. 
\begin{definition}
  The graph  $G = \tuple{V, E}$ is bipartite
  if $V$ partitions into 2 classes, $V =  A \cup B$  and 
  $A\cap B = \emptyset$, 
  such that every edge has ends in different classes. 
 \end{definition}
\begin{definition}
  Let $G = \tuple{A \cup B, E}$ be a bipartite graph. 
  A matching that covers $A$
  is a set of  pairwise non-adjacent edges $M \subseteq E$ such that
  every vertex from $A$ is incident to exactly one edge from $M$.
\end{definition}
We will consider simultaneous matchings on
bipartite graphs(\sbm)~\cite{KEKM08}. 
\begin{definition}
An  \emph{overlapping bipartite graph} is a bipartite graph 
$G = \left\langle A \cup B,E \right\rangle$ and two sets
$S$ and $T$ such that $A = S \cup T$, 
$A\cap B  = \emptyset$, 
and $S \cap T \neq \emptyset$. 
\end{definition}

\begin{definition}
  Let $\tuple{A \cup B, E}$ and $S$, $T$ be an overlapping bipartite graph. 
  A simultaneous matching
  is a set of edges $M \subseteq E$ such that $M \cap (S \times B)$ 
  and $M \cap (T \times B)$ are matchings that cover $S$ and $T$, respectively.
\end{definition}

In the following, we use the convention that 
a set of vertices $P$ is a subset of the partition $A$.
We write $N(P)$ for the neighborhood of $P$,
$P^S = P \cap (S \setminus T)$, $P^T = P \cap (T \setminus
S)$ and $P^{ST} = P \cap S \cap T$.
$\sbm$ problems frequently occur in real world applications like
production scheduling and timetabling. We introduce 
here a simple  exam timetabling problem that 
will serve as a running example. 

\begin{myexample}
  \label{e:intro}
We have 7 exams offered over 5 days and 2 students. 
The first student
has to take the first 5 exams  and the second student
has to take the last 5 exams. Due to
the availability of examiners, not
every exam is offered each day. 
For example,  the first exam 
cannot be on the last day of the week.
Only one exam can be sat each day.
This problem can be encoded as a $\sbm$ problem.
$A$ represents the exams and contains $7$ vertices $X_1$ to $X_7$.
$B$ represents the days and contains the vertices 1 to 5.
$S = [X_1,X_2,X_3,X_4,X_5]$ and 
$T = [X_3,X_4,X_5,X_6,X_7]$.
We connect vertices between $A$ and $B$ to encode the
availability restrictions of the examiners. 
The adjacency  matrix of the 
graph is as follows:
$$
{\scriptsize
\begin{array}{c c|ccccc} 
& & 1 & 2 & 3 & 4 & 5  \\ \hline
\multirow{2}{*}{$A^{S} = S \setminus T$}
&  X_1 & \ast & \ast & \ast & \ast &   \\ 
&  X_2 & \ast & \ast & & & \\ 
\hline \hline
\multirow{3}{*}{$A^{ST} =  S \cap T$}
&  X_3 &  \ast & \ast & \ast &  & \\ 
&  X_4 &  \ast & \ast & \ast &  & \\ 
&  X_5 &  \ast & \ast & \ast &  \ast &  \ast\\ 
\hline \hline
\multirow{2}{*}{$A^{T} =  T \setminus S$}
&  X_6 &  & \ast  & \ast& &  \\
&  X_7 &  & \ast  & \ast& \ast& \ast 
\end{array}
}
$$

Finding a solution for this $\sbm$ problem 
is equivalent to solving the timetabling problem.
\qed
\end{myexample}

\section{Simultaneous Bipartite Matching}
We now consider how to find
a simultaneous matching. 
Unfortunately, this problem is NP-complete in general~\cite{KEKM08}. 
Our contribution here is to identify a necessary
and sufficient condition for the existence of
a simultaneous matching based on an extension of
Hall's theorem~\cite{Hall35}.
We use this to show that
a simultaneous matching on a convex
bipartite graph can be found in polynomial time.
%
%

In the following, let $G'_{(u,v)}$ be the subgraph
of the overlapping bipartite graph $G$
that is induced by 
choosing an edge $(u,v)$ to be in the simultaneous matching. 
If $u \in A^{ST}$ then $G'_{(u,v)} = G - \{u, v\}$.
If $u \in A^S$ (and symmetrically if $u \in
  A^T$) then
  $G'_{(u,v)} = \tuple{V-\{u\}, E \setminus \set{ (u', v) | u' \in
      S}}$. 
  If $M$ is a \sbm in $G'_{(u,v)}$, then $M \cup \{(u,v)\}$ is a \sbm
  in $G$.
 Since the edge $(u,v)$ is implied throughout, we write $G' =
  G'_{(u,v)}$. In addition, we write $N'(P) = N_{G'}(P)$.

\subsection{Extension of Hall's Theorem}
\label{sec:extend-hall}

Hall's theorem provides a necessary and sufficient
condition for the existence of a perfect matching in a bipartite
graph.

\begin{theorem}[Hall Condition~\cite{Hall35}]
  Let $G = \tuple{A \cup B, E}$ such that $A \cap B =
  \emptyset$. There exists a perfect matching iff
  $  |N(P)| \geq |P| $ for $ P \subseteq A$.
\end{theorem}

Interestingly we only need a small adjustment
for simultaneous matching.

%

\begin{theorem}[Simultaneous Hall Condition (\shc)]
  \label{thm:extend-hall}
  Let $G = \left\langle A \cup B,E \right\rangle$ and sets $S$, $T$ be an overlapping bipartite graph.
 There exists a $\sbm$, iff
   $ |N(P)| + |N(P^S) \cap N(P^T)|\geq |P|$  for $ P \subseteq A$. 
\end{theorem}

\begin{proof}

  We prove $\shc$ by induction on
  $|A|$. When $|A|=1$, the statement holds. Let $|A|=k > 1$. 

	If $A^S = \emptyset$ or
  $A^T = \emptyset$ then $\shc$ reduces to the
  condition of Hall's theorem and the statement is true for that reason.
  Hence, we assume $A^S \neq \emptyset$ and $ A^T \neq \emptyset$.
  We show that there is an edge $(u,v)$ that 
  can be chosen for a simultaneous matching and the
  graph $G'_{(u,v)}$ will satisfy $\shc$. 
	Following~\cite{Diestel}, page $37$, we consider two cases. The first case when all subsets of $A$
	satisfy the strict $\shc$, namely, 
	$|N(P)| + |N(P^S) \cap N(P^T)| > |P|$ and the second case when 
	we have an equality.

	\textbf{Case 1}.
  Suppose $|N(P)| + |N(P^S) \cap N(P^T)| > |P|$ for all sets $P \subset A$.
\myOmit{
  Then, we choose an edge $(u,v)$ such that if $u \in A^{ST}$ then
  $v \in N(A) \setminus (N(A^S) \cap N(A^T))$.
  We construct the graph $G'_{(u,v)}$. 
  For any set $P \subset A\setminus\{u\}$, 
  $N'(P) \geq N(P)$ and $N'(P^S) \cap N'(P^T) = 
  N(P^S) \cap N(P^T)$ by the construction of $G'$.
  It holds that  $|N'(P)| + |N'({P^S}) \cap N'({P^T})|$
  $\geq |N(P)| + |N({P^S}) \cap N({P^T})| - 1$
  $\geq |P|$,
  so by the inductive hypothesis 
  there exists a simultaneous matching in $G'$.
}
  As $A^S \neq \emptyset$ we select any edge $(u,v), u \in A^S$ and construct the graph $G'_{(u,v)}$
  (the case $u \in A^T$ is symmetric).
  For any set $P \subset A\setminus\{u\}$ we consider two cases: either $v \notin N(P)$ or $v \in N(P)$. 
  In the first case, the   neighborhood of $P$ is the same in $G$ and $G'$, so 
  the \shc holds for $P$.
  In the case that $v \in N(P)$, then either
  $v$ is a shared neighbor of $P^S$ and $P^T$, which means
  that $N'(P^S) \cap N'(P^T) = N(P^S) \cap N(P^T) - 1$ but
  $N'(P) = N(P)$ by construction,
  or $v$ is a neighbor of $P^S$ but not of $P^T$. Therefore
  $N'(P) \geq N(P) - 1$. But
  $N'(P^S) \cap N'(P^T) = N(P^S) \cap N(P^T)$ by construction.
  In either case, 
$|N'(P)| + |N'({P^S}) \cap N'({P^T})|\geq |N(P)| + |N(P^S) \cap N(P^T)| - 1 \geq |P|$
  %
  for any set $P$ in $G'$. 
  By the inductive hypothesis 
  there exists a simultaneous matching in it.
	
	\textbf{Case 2}.
  Suppose that there exists a set $P \subsetneq A$ such that
  $|N(P)| + |N(P^S) \cap N(P^T)| = |P|$. 
  Let $Q = \tuple{A' \cup B', E'}$
  such that $A' = A \setminus P$, 
  $B' =  B \setminus (N(P^S) \cap N(P^T))$ and 
  \begin{align*}
    E' &=&& \left\{ (u,v) \in E \cap (A' \times B') \mid \right.\\
    &&& (u \in {A^S}' \implies v \notin N(P) \setminus N(P^T)) &\land\\
    &&& (u \in {A^T}' \implies v \notin N(P) \setminus N(P^S)) &\land\\
    &&& (u \in {A^{ST}}' \implies v \notin N(P))     & \left.\right\}
  \end{align*}

  There exists a simultaneous matching in $G - Q$ by
  the inductive hypothesis.  
  We claim that the \shc holds also for
  $Q$. This implies that, by the inductive hypothesis, there exists a
  simultaneous matching in $Q$. 
  Suppose there exists a set $P' \subseteq A'$ that violates the \shc
  in $Q$.


%
%
%
  We denote as $N(P)$ the neighborhood of $P$ in
  $G$ and $N_Q(P')$ as the neighborhood of $P'$ in $Q$.
  We know that the sets $P'$ and $P$ are disjoint.  
  We observe that $N({P}\cup {P}') = N(P)\cup N({P}') $ $
  = N(P)\cup N_Q({P}')$,
  because  $\left(N({P}')\setminus N_Q({P}')\right) \subseteq N(P)$ by 
  construction of $Q$. Moreover 
  $|N(P)\cup N_Q({P}')| = |N(P)| +  |N_Q({P}')| - |N(P) \cap N_Q({P}')|$ $=$
  $|N(P)| +  |N_Q({P}')| - | N(P) \cap (N_Q({P^S}') \cup  N_Q({P^T}') \cup N_Q({P^{ST}}')) |$.
  By construction of $Q$, we have that $ N_Q({P^{ST}}') \cap N({P}) = \emptyset$,
  $ N_Q({P^{S}}') \cap N(P) = N_Q({P^S}') \cap   N(P^{T})$ and
  $ N_Q({P^{T}}') \cap N(P) = N_Q({P^T}') \cap N(P^{S})$.
  Hence,
  $|N(P)\cup N_Q({P}')| = $
  $|N(P)| +  |N_Q({P}')| - |D|$,
  where $D = ( N(P^{S}) \cap N_Q({P^T}')) \cup ( N(P^{T}) \cap N_Q({P^S}'))$.
  Similarly,
  $N({P^S}\cup {P^S}') = N(P^S)\cup N_Q({P^S}')$
  and 
  $N({P^T}\cup {P^T}') = N(P^T) \cup N_Q({P^T}')$.      
  Therefore, 
  $|N(P^S \cup {P^S}') \cap N(P^T \cup {P^T}')| = $
  $|(N(P^S) \cup N_Q({P^S}')) \cap (N(P^T) \cup N_Q({P^T}'))| \leq $
  $|N(P^S)\cap N(P^T)| + |N_Q({P^S}')\cap N_Q({P^T}')| + $
  $|D|$.
%
%
Finally, we have that
  \begin{align*}
    |N(P \cup P')| + |N(P^S \cup {P^S}') \cap N(P^T \cup {P^T}')| &&\leq \\
		|N(P)| +  |N_Q({P}')| - |D| && +\\
		|N(P^S)\cap N(P^T)| + |N_Q({P^S}')\cap N_Q({P^T}')| + |D|& & = \\
    |N(P)| + |N(P^S) \cap N(P^T)| &&+ \\|N_Q(P')| + |N_Q({P^S}') \cap N_Q({P^T}'| && < \\
    |P| + |P'| = |P \cup P'|.
  \end{align*}
Hence $P \cup P'$ violates the \shc in $G$, a
contradiction. Therefore, there exists a simultaneous matching in $Q$.

Let $M$ be a simultaneous matching in $Q$. 
For any edge $(u,v) \in M$, we construct the graph
$R = (G-Q)_{(u,v)}$ and show that
$N_R(P^*) = N_G(P^*)$
and $N_R({P^S}^*) \cap  N_R({P^T}^*) = N({P^S}^*) \cap  N({P^T}^*)$
for any $P^* \subseteq P$.

Let $(u,v)$ be an edge in $M$.
By construction of $Q$, we have that
$v \notin  N({P^S}^*) \cap N({P^T}^*)$.
Hence, the construction of $R$ 
leaves the size
of $N_R({P^S}^*) \cap N_R({P^T}^*)$
the same as in $G$.
Moreover,
$u \in {A^{ST}}' \Rightarrow v \notin N(P^*)$,
$u \in {A^S}' \Rightarrow v \notin N({P^S}^*)$,
$u \in {A^T}' \Rightarrow v \notin N({P^T}^*)$.

Consider the remaining options for the edge $(u,v)$.
If $u \in {A^S}'$ ($u \in {A^T}'$ is similar) then 
$v$ can be in $N({P^T}^*)$, so $v$ is a shared vertex with $N(P^*)$. 
There are two cases to consider: $v \in N({P^T}^*)\setminus  N({P^{ST}}^*)$ and 
$v \in N({P^T}^*) \cap  N({P^{ST}}^*)$. 
In the first case, 
the construction of $R$ leaves the size of $N_R({P^T}^*)$ 
the same as $N({P^T}^*)$
because $u \in {A^S}'$ and can share vertices with ${P^T}^*$. In the latter case,
$N_R({P^{ST}}^*) = N({P^{ST}}^*) \setminus \{v\}$, 
but $N_R({P^T}^*) = N({P^T}^*)$. Hence,
$N_R({P}^*) = N({P}^*)$ in both cases. 
Therefore, the $\shc$ holds for any $P^* \subseteq P$.
\qedhere

\end{proof} 

\begin{myexample}
  \label{e:dc-cons-check}
  In our running example $A = [X_1,X_2,X_3,X_4,X_5,X_6,X_7]$, 
  $A^S = [X_1,X_2]$, $A^{ST} = [X_3,X_4,X_5]$ and $A^T = [X_6,X_7]$ .
  It is easy to check  that the simultaneous Hall condition holds
  for all subsets of the partition $A$.

\end{myexample}

Note that Theorem~\ref{thm:extend-hall} does not
give a polynomial time method to decide if a simultaneous matching
exists. Verifying
that $\shc$ holds requires checking 
the exponential number of subsets of $A$.

\subsection{Removing edges}

To build a propagator, we consider how to detect edges that cannot appear in any
simultaneous matching.


\begin{definition}
Let $G = \left\langle A \cup B,E \right\rangle$ and sets $S$, $T$ be an overlapping bipartite graph. 
A set $P$, $P \subseteq A$, is
\begin{description}
	\item	 \textbf{a simultaneous Hall set} iff  

	\ \ \ \ $|N(P)| + |N(P^S) \cap N(P^T)| = |P|$.
	\item  \textbf{an almost simultaneous Hall set} iff  

	\ \ \ \ $|N(P)| + |N(P^S) \cap N(P^T)| = |P| + 1$.
	\item  \textbf{a loose set} iff  

	\ \ \ \ $|N(P)| + |N(P^S) \cap N(P^T)| \geq |P| + 2$.
\end{description}
\end{definition}

\begin{theorem}
  \label{thm:dc-over-alldiff}
	$G = \left\langle A \cup B,E \right\rangle$ and sets $S$, $T$ be an overlapping bipartite graph. 
  Each edge $(u,v)$, $u \in A$ and $v \in B$ can be extended to a matching that covers $S$ and $T$ iff
\begin{enumerate}
	\item \textbf{for each set $P$}:
\begin{enumerate}
	\item  \label{t:e:unstable}$|N(P)| + |N(P^S) \cap N(P^T)| \geq |P|$ 
\end{enumerate}	 	
	\item \textbf{for each simultaneous Hall set $P$}:	
\begin{enumerate} 
	\item \label{t:e:out_all} if $u \notin P $ then $ v \notin (N(P^S) \cap N(P^T))$		
	\item \label{t:e:out_side_one} if $u \in S \setminus (T \cup P)$  then $ v\notin \left(N(P) \setminus N(P^T)\right)$
	\item \label{t:e:out_side_two} if $u \in T \setminus (S \cup P)$  then $ v \notin \left(N(P) \setminus N(P^S)\right)$	 
	\item \label{t:e:out_inter}  if $u  \in (S \cap T) \setminus P$  then $ v \notin  N(P)$	
\end{enumerate}
	
	\item \textbf{for each almost simultaneous Hall set $P$}:	
\begin{enumerate}
	\item  \label{t:e:out_inter_almost}  if $u \in (S\cap T) \setminus P$ then $v \notin N(P^S) \cap N(P^T)$ 
\end{enumerate}	
\end{enumerate}

\end{theorem}

\begin{proof}

  \textbf{Soundness}. The soundness of Rule~\ref{t:e:unstable} follows from Theorem~\ref{thm:extend-hall}.
	Let $(u,v)$ be an edge that we want to extend to a matching. Suppose that $(u,v)$ 
	violates one of the rules for a $\shs$ or an $\ashs$ $P$ in $G$. 
  We show that if $(u,v)$ is selected to be in a matching, then  $P$ fails $\shc$ in $G_{(u,v)}'$. 

\textbf{	Rule~\ref{t:e:out_all}: } 
	If $(u,v)$ violates Rule~\ref{t:e:out_all} for a $\shs$ $P$  then	
	$|N'(P^S) \cap N'(P^T)| = |N(P^S) \cap N(P^T)|-1$ and $N'(P) = N(P)$, so the \shc is violated for $P$  in $G'$.
	
\textbf{	Rule~\ref{t:e:out_side_one}:} 
	If $(u,v)$ violates Rule~\ref{t:e:out_side_one} for a $\shs$ $P$ then	
	$|N'(P)| = |N(P)|-1$ and $|N'(P^S) \cap N'(P^T)| = |N(P^S) \cap N(P^T)|$ so the \shc is violated for $P$  in $G'$.

\textbf{	Rule~\ref{t:e:out_side_two}:} Symmetric to Rule~\ref{t:e:out_side_one}.
	
\textbf{	Rule~\ref{t:e:out_inter}:} 
	If $(u,v)$ violates Rule~\ref{t:e:out_inter} for a $\shs$ $P$  then	
	$|N'(P)| = |N(P)|-1$ so the \shc is violated for $P$ in $G'$.


	\textbf{Rule~\ref{t:e:out_inter_almost}:} 
		If $(u,v)$ violates Rule~\ref{t:e:out_inter_almost} for an $\ashs$ $P$  then	
	$|N'(P)| = |N(P)|-1$ and $|N'(P^S) \cap N'(P^T)| = |N(P^S) \cap N(P^T)|-1$,
        so $|N'(P)| + |N'(P^S) \cap N'(P^T)| = |P|-1$ and the \shc is violated for $P$ in $G'$.
        

	\textbf{Completeness}. 
	Second, we show that Rules~\ref{t:e:out_all}-~\ref{t:e:out_inter_almost} are \emph{complete}.	
	We will show that we can use any edge $(u,v)$ in a maching
        by showing that the graph $G'_{(u,v)}$ satisfies the $\shc$, 
        thus has a $\sbm$.

        Suppose there is a set $P$ that violates the \shc in $G'$ but
        not in $G$ so that
        \begin{equation} \label{t:eq:violation}
          |N'(P)|  +  |N'(P^S ) \cap N'(P^T)| < |P|
	\end{equation}

        and
        \begin{equation} \label{t:eq:non-violation}
          |N(P)|  +  |N(P^S ) \cap N(P^T)| \geq |P|
	\end{equation}

        Note that $N'(P) = N(P) \setminus \set{v}$ and $N'(P^S ) \cap
        N'(P^T) = N(P^S ) \cap N(P^T) \setminus \set{v}$. Hence,
        $|N'(P)| \geq |N(P)| - 1$ and $|N'(P^S ) \cap
        N'(P^T)| \geq |N(P^S ) \cap N(P^T)| - 1$.

	There are three cases to consider for $P$ in $G$, when
	$P$ is  a loose set, a $\shs$ and an $\ashs$ in $G$.
	These cases are similar, so we consider only the most 
	difficult case. 
\myOmit{
\emph{$\bullet \ \ P$ is a loose set in $G$}. 

Then $|N'(P)| + |N'(P^S ) \cap N'(P^T)| \geq |N(P)| + |N(P^S ) \cap
N(P^T)| - 2 \geq |P|$, so~(\ref{t:eq:violation}) 
and~(\ref{t:eq:non-violation}) cannot both be true.

\emph{$\bullet \ \ P$ is a $\shs$ in $G$}.

If $u \in A^{ST}$ then  $v \notin N(P)$ by Rule~\ref{t:e:out_inter}.
Hence $N'(P) = N(P)$ and $N'(P^S) \cap N'(P^T) = N(P^S) \cap N(P^T)$, 
so~(\ref{t:eq:violation}) 
and~(\ref{t:eq:non-violation}) cannot both be true.
If $u \in A^{S}$ ($u \in A^T$ is symmetric) then $v \notin N(P^S) \cap
N(P^T)$ by Rule~\ref{t:e:out_all} and $v \notin \left(N(P) \setminus
  N(P^T)\right)$ by Rule~\ref{t:e:out_side_one}.  Hence, $v \in
N(P'^T) \setminus N(P'^S)$. So again $N'(P) = N(P)$ and $N'(P^S) \cap
N'(P^T) = N(P^S) \cap N(P^T)$, which means that~(\ref{t:eq:violation}) 
and~(\ref{t:eq:non-violation}) cannot both be true.
}
Let $P$ be an $\ashs$ in $G$. 
If $u \in A^{ST}$ then $v \notin N(P^S) \cap N(P^T)$ by
Rule~\ref{t:e:out_inter_almost}.  Hence $N'(P^S) \cap N'(P^T) =
N(P^S) \cap N(P^T )$, so $|N'(P)| + |N'(P^S ) \cap N'(P^T)| \geq
|N(P)| + |N(P^S ) \cap N(P^T)| - 1 \geq |P|$ and
therefore~(\ref{t:eq:violation}) 
and~(\ref{t:eq:non-violation}) cannot both be true.

If $u \in A^{S}$ ($u \in A^{T}$ is symmetric) then $v \in N(P^S )
\cap N(P^T )$ or its complement.  In the first case $N'(P) = N(P)$,
while in the second $N'(P^S) \cap N'(P^T) = N(P'^S) \cap N(P'^T )$. In
both cases, $|N'(P)| + |N'(P^S ) \cap N'(P^T)| \geq |N(P)| + |N(P^S )
\cap N(P^T)| - 1 \geq |P|$ so~(\ref{t:eq:violation}) 
and~(\ref{t:eq:non-violation}) cannot both be true.
\myOmit{
        All sets in $G'$ can be partitioned into two groups depending on whether
	the vertex $u$ belongs to $P$ or not. We consider these two cases.

	\textbf{Case 1: }Suppose  there is a set $P$, $u \in P$,  such that the set $P' = P \setminus \{u\}$ that
	 violates $\shc$ 	so that
\begin{equation} \label{t:eq:violation}
	|N(P')\setminus \{v\}|  +  |N(P'^S ) \cap N(P'^T)\setminus \{v\}| < |P'|
	\end{equation}	
	
	Consider the set $P'$ before we remove the edge $(u,v)$. By the statement of 
	Theorem $\shc$ holds for any set including the set $P'$. 
\begin{equation} \label{t:eq:smaller_set_holds}
	|N(P') |  +  |N(P'^S  ) \cap N(P'^T  ) | \geq |P'| 
\end{equation}		
	
Hence, the removal of the vertex $v$ from neighborhood of $P'$ makes this set to violate $\shc$.					
Next, we show that there is no such edge $(u,v)$ exist.

We have to consider three cases for the set $P'$.

\emph{$\bullet \ \ P'$ is a loose set}.

In this case the removal of a vertex $v$ can decrease 
$|N(P')|  +  |N(P'^S ) \cap N(P'^T ) |$ be at most $2$. Hence,  Equation~\eqref{t:eq:violation}
can not be satisfied.

\emph{$\bullet \ \ P'$ is a $\shs$}.

If $u \in P^{ST}$ then  $v \notin N(P')$ by Rule~\ref{t:e:out_inter}.
Hence the removal of $v$ does not change LHS of Equation~~\eqref{t:eq:smaller_set_holds}.
If $u \in P^{S}$ then  $v \notin N(P'^S) \cap N(P'^T)$ by Rule~\ref{t:e:out_all}
and $v \notin  \left(N(P) \setminus N(P^S)\right)$ by Rule~\ref{t:e:out_side_one}.
Hence, $v \in  N(P'^S) \setminus N(P'^T)$ and the removal of $v$ does not change LHS of Equation~\eqref{t:eq:smaller_set_holds}.

\emph{$\bullet \ \ P'$ is an $\ashs$}.

If $u \in P^{ST}$ then $v \notin N(P'^S) \cap N(P'^T)$ by Rule~\ref{t:e:out_inter_almost}.
 Hence the removal of $v$ does not change $|N(P'^S ) \cap N(P'^T )|$ and might
	reduce $|N(P')|$ by at most one. Hence,  Equation~\ref{t:eq:violation} can not be satisfied.

	If $u \in P^{S}$ ($u \in P^{T}$ is symmetric)  then $v \in N(P'^S ) \cap N(P'^T )$ or its complement.
	In the both case the removal of $v$ can reduce the size of the LHS of Equation~\ref{t:eq:smaller_set_holds} by at most one.
	
\myOmit{	
	Finally, suppose that $u \in S\setminus T$. In this case we do not have any restriction 
	on the vertex $v$. Consider 2 cases. If $v \in N(P'^S ) \cap N(P'^T )$ then the removal 
	of the vertex $v$ can reduce the size of $N(P'^S ) \cap N(P'^T )$ by one and  can not reduce the
	size of $N(P')$.
	If $v \in N(P') \setminus (N(P'^S ) \cap N(P'^T ))$ then the removal 
	of the vertex $v$ can not reduce the size of $N(P'^S ) \cap N(P'^T )$  and can  reduce the
	size of $N(P')$ by at most one.
}

	\textbf{Case 2: }Suppose  there is a set $P$, $u \notin P$ and $v \in N(P)$ that
	 violates $\shc$. We again have to consider three cases for the set $P$. If the
	 set $P$ is a loose set or an $\ashs$ that it is easy to check that a removal of the
	 vertex $v$ does not make $P$ to violate $\shc$. Consider the last case when $P$ is a $\shs$.
	  
\myOmit{
	\emph{$\bullet \ \ P$ is a loose set}. 		
	If $P$ is a  loose set then a removal of a vertex $v$ can reduce that
	LHS on at most $2$. Hence, the  simultaneous Hall condition holds. 	
	\paragraph{\textbf{Almost simultaneous Hall set}, $u \in S \cap T$, $u \notin P$}.
	If $u \in S \cap T$ then $v \notin  N(P^S) \cap N(P^T)$ for any
	almost simultaneous Hall set(~\ref{t:e:out_inter_almost}). Hence, the LHS can be only reduced by one
	and the simultaneous Hall condition holds. 	
	\paragraph{\textbf{Almost simultaneous Hall set}, $u \in S \setminus T$, $u \notin P$}.
	Suppose $u \in S \setminus T$ and $v \in N(P^{S}) \cap N(P^{T})$. In this case the removal of the 
	vertex $v$ reduces the size of $N(P^S) \cap N(P^T)$ by at most one
	and will not change $N(P)$.
	Suppose $u \in S \setminus T$ and $v \in N(P) \setminus (N(P^{S}) \cap N(P^{T}))$. In this case the removal of the 
	vertex $v$ will reduces $N(P)$ by at most one and will not change $N(P^S) \cap N(P^T)$.
	\paragraph{\textbf{Almost simultaneous Hall set}, $u \in T \setminus S$, $u \notin P$}.
	This case is similar to the previous case.
}	
	
		
	If $u \in A^{ST}$ then $ v \notin N(P)$ by Rule~\ref{t:e:out_inter} and $\shc$ holds.
	If $u \in A^S$ ($u \in A^{T}$ is symmetric) then
	$v \notin  N(P^S) \cap N(P^T)$ by Rule~\ref{t:e:out_all}
	and  $ v \notin  \left(N(P) \setminus N(P^T)\right)$ by Rule~\ref{t:e:out_side_one}.
	Hence, $v $ has to be in $N(P^T) \setminus N(P^S)$. 
	In this case the removal of the vertex $v$ does not change $|N(P)|$ and $| N(P^S) \cap N(P^T)|$
	and $\shc$ holds.

\myOmit{
	\paragraph{\textbf{Simultaneous Hall set, $u \in S \setminus T$, $u \notin P$}}.
	If $u \in S \setminus T$ then $v \notin  N(P^S) \cap N(P^T)$(~\ref{t:e:out_all})
	and  $ v \notin  \left(N(P) \setminus N(P^T)\right)$(~\ref{t:e:out_side_one}) for any simultaneous Hall set.  
	Hence, the removal of $v$ does not change the size of $N(P^S) \cap N(P^T)$.	
	Moreover, the only possible option for $v$ is to be in $N(P^T)  \setminus N(P^S)$.
	In this case the removal of the vertex $v$ does not change the value $N(P)$.
	Hence, the  simultaneous Hall condition holds. 	
	\paragraph{\textbf{Simultaneous Hall set}, $u \in T \setminus S$}.		
	This case is similar to the previous case.
	}
}
\end{proof}

\begin{myexample}
  \label{e:dc-prop}
Consider again our running example. We show that
Rules~\ref{t:e:out_all}-\ref{t:e:out_inter_almost} remove every edge 
that can not be extended to a  matching.
Consider the set $ P = \{X_2,X_3,X_4,X_6\}$. This is a $\shs$ as $N(P) = \{1,2,3\}$, $N(P^S) \cap N(P^T) = \{2\}$
and $4 = |P| = $ $|N(P)| + |N(P^S) \cap N(P^T)| = 4$. Hence, by Rule~\ref{t:e:out_inter}
we prune $1,2,3$ from $X_5$ and by Rule~\ref{t:e:out_all}
we prune $2$ from $X_1$ and $X_7$.
Now consider the set $ P = \{X_2,X_3,X_6\}$. This is 
an $\ashs$. By Rule~\ref{t:e:out_inter_almost}
we prune $2$ from $X_4$.  The set $ P = \{X_2,X_4,X_6\}$
is also an $\ashs$ and, by Rule~\ref{t:e:out_inter_almost},
we prune $2$ from $X_3$.
\myOmit{
$$
{\scriptsize
\begin{array}{c c|ccccc} 
& & 1 & 2 & 3 & 4 & 5  \\ \hline
\multirow{2}{*}{$A^S = S \setminus T$}
 & X_1 & \ast &  & \ast & \ast &   \\ 
 & X_2 & \ast & \ast & & & \\ 
\hline \hline
\multirow{3}{*}{$A^{ST} = S \cap T$}
 & X_3 &  \ast &  & \ast &  & \\ 
 & X_4 &  \ast &  & \ast &  & \\ 
 & X_5 &   &  &  &  \ast &  \ast\\ 
\hline \hline
\multirow{2}{*}{$A^T = T \setminus S$}
 & X_6 &  & \ast  & \ast& &  \\
 & X_7 &  &   & \ast & \ast& \ast 
\end{array}
}
$$
}
Next consider the set $P = \{X_3,X_4\}$ which is 
a $\shs$. By Rules~\ref{t:e:out_side_one} and~\ref{t:e:out_side_two}
we prune $1,3$ from $X_1$, $X_2$,  $X_6$ and $X_7$. 
Now, $\{X_1\}$ is a $\shs$ and $4$ is pruned from
$X_5$ by Rule~\ref{t:e:out_inter}. Finally, from the simultaneous Hall
set $\{X_5\}$, we prune $5$ from  $X_7$ using Rule
~\ref{t:e:out_side_two} and we are now at the fixpoint.

$$
{\scriptsize
\begin{array}{cc|ccccc} 
& & 1 & 2 & 3 & 4 & 5  \\ \hline
\multirow{2}{*}{$A^S = S \setminus T$}
 & X_1 &  &  &  & \ast &   \\ 
 & X_2 &  & \ast & & & \\ 
\hline \hline
\multirow{3}{*}{$A^{ST} = S \cap T$}
 & X_3 &  \ast &  & \ast &  & \\ 
 & X_4 &  \ast &  & \ast &  & \\ 
 & X_5 &   &  &  &   &  \ast\\ 
\hline \hline
\multirow{2}{*}{$A^{T} = T \setminus S$}
 & X_6 &  & \ast  & & &  \\
 & X_7 &  &   & & \ast& 
\end{array}
}
$$

\end{myexample}

\section{The overlapping $\alldiff$ constraint}
\label{ss:oalldiff}


We now uses these results to build a propagator. 

\begin{definition}
   $\oalldiff ([X],S,T)$ 
   where $S \subseteq X$, $T \subseteq X$, $S \cup T = X$ holds iff  
   $\alldiff ( S )$ and $\alldiff ( T )$ hold simultaneously .
\end{definition}

Enforcing $DC$ on the $\oalldiff$ constraint is $NP$-hard \cite{complex}. 
We consider instead enforcing just $BC$. 
This relaxation is equivalent to 
the simultaneous matching problem on a bipartite convex variable-value graph.
Our main result is an algorithm that enforces
$BC$ on the $\oalldiff$ constraint in $O(nd^3)$ time.
\myOmit{
The $\oalldiff$ constraint can be encoded as a simultaneous 
bipartite matching problem in a variable-value graph the same way as the 
$\alldiff$ constraint can be encoded as a perfect matching problem. 
Therefore, we can use Theorem~\ref{thm:extend-hall} and 
Rules~\ref{t:e:out_side_one}--~\ref{t:e:out_inter_almost} 
to enforce domain consistency for the $\oalldiff$ constraint.
However, it is not practical to consider all possible sets of variables because
there is an exponential number of them. Based on these rules, we propose an 
a $BC$ propagator for the $\oalldiff$ constraint.
}
%
%
The algorithm is based on the decomposition
of the $\oalldiff$ constraint into a set of arithmetic constraints
derived from Rules~\ref{t:e:out_side_one}--\ref{t:e:out_inter_almost}. It is
inspired by a decomposition of $\alldiff$ \cite{bknqw09}. 
As there, we introduce Boolean variables $a_{ilu}$, $b_{il}$ 
to represent whether
$X_i$ takes a value in the interval $[l,u]$
and the variables $C^S$, $C^{ST}$ and $C^T$ to represent
bounds on the number of variables from $S\setminus T$, $T\setminus S$ and $S \cap T$  that may take
values in the interval $[l,u]$.
We introduce the following  
set of constraints for $1 \leq i \leq n$, $1 \leq l \leq u \leq d$ and $u-l < n$: 

\begin{align}
   b_{il}=1  \iff  X_i \leq l & & \;   \label{eqn::firstOverlappingAlldiff} \\
   a_{ilu}= 1  \iff (b_{i(l-1)}=0 \wedge b_{iu}=1)  & & \;  \label{eqn::chaneling} \\
   C_{lu}^{ST}  = \sum_{i \in S \cap T} a_{ilu} & &\;  \label{eqn::channelST}\\
   C_{lu}^{S}  = \sum_{i \in S \setminus T} a_{ilu}& &  \;  \label{eqn::channelS}\\
   C_{lu}^{T}  = \sum_{i \in T \setminus S} a_{ilu}& & \;   \label{eqn::channelT}\\ 
   C_{1u}^{ST}  = C_{1l}^{ST} + C_{(l+1)u}^{ST} & &  \; \label{eqn::pyramidST}\\
   C_{lu}^{ST} + C_{lu}^{S} \leq u - l +1 & &\; \label{eqn::overlapping-alldiff-S}\\ 
   C_{lu}^{ST} + C_{lu}^{T} \leq u - l +1 & &\;  \label{eqn::lastOverlappingAlldiff}
\end{align}

We also introduce a dummy variable $C^{ST}_{10}=0$ to simplify the following lemma and theorems.
\begin{lemma}
  \label{l:chaining}
  Consider a sequence of values $v_1, v_2, \ldots, v_{k}$. Enforcing $BC$ on~\eqref{eqn::pyramidST}
  ensures $ub(C^{ST}_{1v_k-1}) \leq \sum_{i : v_i < v_{i+1}}
  ub(C_{v_{i}v_{i+1}-1}^{ST}) - \sum_{i : v_i > v_{i+1}}
  lb(C^{ST}_{v_{i+1}v_i-1})$.
\end{lemma}
\begin{proof}
  For every $i$ such that $v_{i} < v_{i+1}$,
  constraint~\eqref{eqn::pyramidST} ensures that
  $ub(C_{1v_{i+1}-1}^{ST}) \leq ub(C^{ST}_{1v_{i}-1}) +
  ub(C^{ST}_{v_{i}v_{i+1}-1})$. For every $i$ such that $v_i >
  v_{i+1}$ constraint~\eqref{eqn::pyramidST} ensures that
  $ub(C^{ST}_{1v_{i+1}-1}) \leq ub(C^{ST}_{1v_{i}-1}) -
  lb(C^{ST}_{v_{i+1},v_{i}-1})$. The left side of each inequality can
  be substituted into the right side of another inequality until one
  obtains $ub(C^{ST}_{1v_k-1}) \leq \sum_{i : v_i < v_{i+1}}
  ub(C^{ST}_{v_{i}v_{i+1}-1}) - \sum_{i : v_i > v_{i+1}}
  lb(C^{ST}_{v_{i+1}v_i-1})$.
\end{proof}

\begin{theorem}
\label{t:bc-disentail}
  Enforcing $BC$ on 
  \eqref{eqn::firstOverlappingAlldiff}-\eqref{eqn::lastOverlappingAlldiff}
  detects bound disentailment of \oalldiff in $O(nd^2)$ time 
  but does not enforce $BC$ on \oalldiff .
  \end{theorem}

\begin{proof}
	First we derive useful upper bounds for the variables $C^{ST}_{lu}$.
	Consider a set $P$ and an interval $[a,b]$ such that $N(P) = [a,b]$.
  Let $[c_1, d_1] \cup \ldots \cup [c_k, d_k]$ be a set of intervals
  that tightly contain variables from $P^S$ so that $\forall i,
  [c_i,d_i] \in N(P^S)$, $[e_1, f_1] \cup \ldots \cup [e_m, f_m]$ be a
  set of intervals that tightly contain variables from $P^T$ so that
  $\forall i, [e_i,f_i] \in N(P^T)$, and $I_1 \cup \ldots \cup I_p$
  are intersection intervals between intervals $[c_i,d_i]$ and
  $[e_i,f_i]$, i.e., $\forall i, I_i \in N(P^S) \cap N(P^T)$.
  
  We first remove all intervals $[c_i,d_i]$ ($[e_i,f_i]$) that are completely 
  inside an interval $[e_j,f_j]$ ($[c_j,d_j]$).     
  For any of these intervals $[c_i,d_i]$ ($[e_i,f_i]$ is similar) there exists
  an intersection interval $I_j$ such that $I_j = [c_i,d_i]$. 
  We denote the set of removed intervals $RI$.
  The remaining intervals are $\{[c_1,d_1], \ldots,
  [c_{k'},d_{k'}]\}$, $[e_1,f_1] \ldots, [e_{m'},f_{m'}] \}$ and
  $I_1,\ldots, I_{p'}$.  For any interval $[c_i, d_i]$,
  \eqref{eqn::overlapping-alldiff-S} ensures that
  $ub(C_{c_id_i}^{ST}) \leq d_i - c_i + 1 -
  lb(C^S_{c_id_i})$. Similarly for an interval $[e_i, f_i]$, we have
  $ub(C^{ST}_{e_if_i}) \leq f_i - e_i + 1 - lb(C^T_{e_if_i})$.

  We sort the union of remaining intervals $[c_i, d_i]$ and $[e_i,
  f_i]$ by their lower bounds and list them as semi-open intervals
  $[g_1, g_2), [g_3, g_4), \ldots, [g_{k^\prime + m^\prime -1},
  g_{k^\prime + m^\prime})$. Using the sequence $a, (g_1, g_2, \ldots,
  g_{k^\prime + m^\prime}, b+1, a)^x$ where $(.) ^x$ indicates a
  repetition of $x$ times the same sequence, Lemma~\ref{l:chaining}
  provides the inequality $ub(C_{1a-1}^{ST}) \leq ub(C_{1a-1}^{ST}) + x
  (ub(C_{ag_1-1}^{ST}) + ub(C_{g_{k^\prime + m^\prime}b}^{ST}) + \sum_{i=1}^{k^\prime} ub(C_{c_id_i}^{ST}) +
  \sum_{i=1}^{m^\prime} ub(C_{e_if_i}^{ST}) - lb(C_{a,b}^{ST}))$.
  Substituting the inequalities that we already defined,
  we obtain $ub(C_{1a-1}^{ST}) \leq a-1 + x ( b - a + 1 +
  \sum_{i=1}^{k^\prime} lb(C_{c_id_i}^{S}) - \sum_{i=1}^{m^\prime}
  lb(C_{e_if_i}^{T}) + \sum_{i=1}^{p^\prime}|I_i| - lb(C_{a,
    b}^{ST}))$.
  
  
  For any removed interval $I_j \in RI$ we have $(|I_j| -  lb(C^S_{min(I_j)max(I_j)})) \geq 0$
  or $(|I_j| -  lb(C^T_{min(I_j)max(I_j)})) \geq 0$.
  We reintegrate all removed intervals into the inequation to get 
  $ub(C_{1a-1}^{ST}) \leq a - 1 + x(b - a + 1 -
  \sum_{i=1}^{k} lb(C_{c_id_i}^{S}) - \sum_{i=1}^{m}
  lb(C_{e_if_i}^{T}) + \sum_{i=1}^{p}|I_i| - lb(C_{a
    b}^{ST}))$.
  Note that $ \sum_{i=1}^p |I_i| = |N(P^S) \cap N(P^T)|$, $lb (C^S_{ab}) = \sum_{i=1}^k lb(C^{S}_{c_id_i})$ 
  and  $lb (C^T_{ab}) = \sum_{i=1}^m lb(C^{T}_{e_if_i})$. Hence 
  \begin{eqnarray*}
  ub(C^{ST}_{1a-1})  \leq a -1 + x(b - a + 1 - lb (C^S_{ab}) - lb (C^T_{ab}) \\
   + |N(P^S) \cap N(P^T)| - lb(C_{ab}^{ST})) \hspace{8mm} \mbox{(*)}
  \end{eqnarray*}

  \textbf{Bound disentailment.} Suppose, for the purpose of contradiction, that \oalldiff is bound disentailed and that constraints \eqref{eqn::firstOverlappingAlldiff}-\eqref{eqn::lastOverlappingAlldiff} are bound consistent.     
  Then, there exists a set $P$, such that $N(P)$ is an interval and $|N(P)| +
  |N(P^S) \cap N(P^T)| < |P|$.   As $P$ fails $\shc$, it holds that
    $lb(C_{ab}^{ST}) + lb(C_{ab}^{S}) +
    lb(C_{ab}^{T}) \geq |P|>     
    |N(P)| + |N(P^S) \cap N(P^T)| = b- a + 1 + |N(P^S) \cap N(P^T)|$    
   or 
    $lb(C_{ab}^{ST}) \geq  b- a + 2  - lb(C_{ab}^{S}) -  lb(C_{ab}^{T}) + |N(P^S) \cap N(P^T)|$. 
  Substituing the last inequality in~(*)
   gives $ ub(C^{ST}_{1a-1})  \leq a - 1 - x $.
  Choosing a large enough value for $x$ (say $a$) gives
  the contradiction $ub(C^{ST}_{1a-1}) < 0$.
   
  \textbf{Bound consistency.}
  To show that this decomposition does not enforce $BC$, 
  consider the conjunction of \alldiff($[X_1, X_2,
  X_3]$) and \alldiff($[X_2, X_3, X_4]$) with $D(X_1) = [2, 3]$,
  $D(X_2) = [2,4]$, $D(X_3) = [1,3]$, $D(X_4) = [1,2]$.
   Enforcing BC on 
  \eqref{eqn::firstOverlappingAlldiff}-\eqref{eqn::lastOverlappingAlldiff}
  does not remove the bound inconsistent value
  $X_2=2$. 
  
  \textbf{Complexity.}
  There are $O(nd)$ constraints~\eqref{eqn::firstOverlappingAlldiff}
  that can be invoked $O(d)$ times at most.
  There are $O(nd^2)$ constraints~\eqref{eqn::chaneling}
  that can be invoked $O(1)$ times. 
  There are $O(d^2)$ constraints~\eqref{eqn::pyramidST} 
  that can be invoked $O(n)$ times. 
  There are $O(d^2)$ constraints~\eqref{eqn::channelST}--~\eqref{eqn::channelT}
  that can be propagated in $O(n)$. 
  The remaining constraints take $O(nd^2)$ to propagate. 
  The total time complexity is $O(nd^2)$.          
\end{proof}

It follows immediately that the simultaneous matching
problem is polynomial on bipartite convex graphs.

\begin{theorem}
	A simultaneous matching can be found in polynomial time
on an overlapping convex bipartite graph. 
\end{theorem}

Next, we present an algorithm to enforce $BC$.
We show that constraints~\eqref{eqn::firstOverlappingAlldiff}--\eqref{eqn::lastOverlappingAlldiff}
together with  the following two constraints enforce all but one of the rules from Theorem~\ref{thm:dc-over-alldiff}.
\begin{align}
   C_{1u}^{T}  = C_{1l}^{T} + C_{(l+1)u}^{T} & & \; 1 \leq l  \leq u \leq d \label{eqn::pyramidT-T}\\
   C_{1u}^{S}  = C_{1l}^{S} + C_{(l+1)u}^{S}& & \; 1 \leq l \leq u \leq d \label{eqn::pyramidS-S}
\end{align}

\begin{theorem}
\label{t:pruning-shs}
 Constraints
  \eqref{eqn::firstOverlappingAlldiff}-\eqref{eqn::pyramidS-S}
  enforce Rules~\ref{t:e:out_all}--~\ref{t:e:out_inter}. 
\end{theorem}

\begin{proof}[Proof Sketch]
Based on Lemma~\ref{l:chaining}, similar to the proof of Theorem~\ref{t:bc-disentail}, 
we show that all intervals that contain variables from a $\shs$ $P$ become saturated intervals, so that
the lower bounds of the corresponding variables  $C^{ST}$,
$C^{S}$ and $C^{T}$  equal 
to their upper bounds. Hence, these values are pruned from domains of 
variables outside the set $P$.
\myOmit{
Suppose $P$ is a $\shs$, $N(P) = [a,b]$.
Let $[c_i,d_i]$, $[e_i,f_i]$ and $I_i$
be the same sets of intervals as defined in  the proof of Theorem~\ref{t:bc-disentail}.
Let $[g_1, h_1] \cup \ldots \cup [g_s, h_s]$
be a set of intervals that tightly contain variables from $P^{ST}$.
Using the chaining argument we can show that  
all this intervals because saturated, so that 
$lb(C^{ST}_{g_i,h_i}) = ub(C^{ST}_{g_i,h_i})$, $i=1,\ldots,s$,
$lb(C^S_{c_i,d_i}) = ub(C^S_{c_i,d_i})$, $i=1,\ldots,k$,
$lb(C^T_{e_i,f_i}) = ub(C^T_{e_i,f_i})$, $i=1,\ldots,m$,
$lb(C^S_{I_i}) = lb(C^T_{I_i}) = ub(C^S_{I_i}) = ub(C^T_{I_i}) = |I_i|$, $i=1,\ldots,m$.
It is easy to see that these saturated intervals cause the removal of their values from 
the rest of the variable domains and all rules are enforced.

We use the chaining argument from the proof of Theorem~\ref{t:bc-disentail}. 
We only sketch the proof because it is similar to the proof of Theorem~\ref{t:bc-disentail}. Here we 
only show one  step of the procedure that proves  saturation of the intersection intervals $I_p$.
Let $f_m = b$, $ e_m < d_k$ and $I_p$ be $[e_m,d_k]$. Consider the interval $[a,d_k]$ .  
By \eqref{eqn:est_upper_bound} and the fact that all variables in
$P^S \cup (P^T \setminus \bigcup_{X_i \in P^T} \{X_i \in [e_m,b]\})$ 
are less than or equal to $d_k$,
we get
$ub(C^{ST}_{ad_k})  \leq d_k-a+1 - lb (C^S_{ab}) - lb (C^T_{ab}) + |N(P^S) \cap N(P^T)| + (lb(C_{e_mb}) - |I_p|)$.
By the constraint $C^{ST}_{ad_k} + C^{ST}_{(d_k+1)b} = C^{ST}_{ab}$ 
and
$lb(C_{ab}^{ST}) =  b- a + 1  - lb(C_{ab}^{S}) -  lb(C_{ab}^{T}) + |N(P^S) \cap N(P^T)|$
we get that
$lb(C^{ST}_{(d_k+1)b})  \geq  b - d_k  - (lb(C_{e_mb}) - |I_p|)$.
By \eqref{eqn::lastOverlappingAlldiff}, 
$ub(C^{T}_{(d_k+1)b})  \leq lb(C_{e_mb}) - |I_p|$.
By \eqref{eqn::pyramidT-T},
$lb(C^{T}_{e_md_k})  \geq |I_p|$. As $I_p = [e_m,d_k]$, we get that 
$ub(C^T_{I_p}) = lb(C^T_{I_p}) = |I_p|$.

Note that if there is a $\shs$ P inside the interval $[a,b]$ then 
a union of  intersection intervals saturated to their length
constitutes the set $N(P^S) \cap N(P^T)$.}
\end{proof}

\myOmit{
We observe 
that constraints \eqref{eqn::firstOverlappingAlldiff}-\eqref{eqn::lastOverlappingAlldiff}
give us an indication that an interval of values may contain a $\shs$ or $\ashs$.
\begin{corollary}
\label{c:detection}
  If $P$ is a $\shs$, $N(P) =[a,b]$ then $lb(C^{ST}_{ab}) = ub(C^{ST}_{ab})$. 
  If $P$ is an $\ashs$, $N(P) =[a,b]$ then $lb(C^{ST}_{ab}) + 1 \geq ub(C^{ST}_{ab})$. 
\end{corollary}
\begin{proof}
  This follows from~\eqref{eqn:est_upper_bound}.
\end{proof}
}

\begin{theorem}
	\label{t:pruning-ashs}
	Suppose constraints
        \eqref{eqn::firstOverlappingAlldiff}-\eqref{eqn::pyramidS-S} together with 
$C_{lu}^{ST}  = C_{lk}^{ST} + C_{(k+1)u}^{ST}$, $2 \leq l \leq k \leq u \leq d$
        have reached
	their fixpoint. Rule~\ref{t:e:out_inter_almost} can now be enforced in $O(nd^3)$ time. 	
\end{theorem}
\begin{proof}[Proof Sketch]
Let $P$  is an $\ashs$,  $N(P) = [a,b]$. Similar to the proof
of Theorem~\ref{t:bc-disentail}, we can obtain that $lb(C^{ST}_{ab}) + 1 \geq ub(C^{ST}_{ab})$. 
Hence, we can identify intervals, that might contain an $\ashs$ $P$.  
Next, we observe that if we add a dummy variable $Z$, $D(Z) = [a,b]$ to the set $P$
so that  $P' = P \cup \{Z\}$, $Z \in {P^{ST}}'$ then $P'$ is a $\shs$.
This allows us to identify the set $N(P^S) \cap N(P^T)$ by simulating constraints~\eqref{eqn::firstOverlappingAlldiff}-\eqref{eqn::pyramidS-S}
inside the interval $[a,b]$ taking into account the variable $Z$.
\myOmit{ 
Corollary~\ref{c:detection} shows that for each $\ashs$ $P'$, $N(P') = [a,b]$ we get that $lb(C^{ST}_{ab}) + 1 \geq ub(C^{ST}_{ab})$. 
We omit the case where $lb(C^{ST}_{ab}) = ub(C^{ST}_{ab})$ because this case is either straightforward
or can be reduced to the second case.
%
If  $lb(C^{ST}_{ab}) + 1 \geq ub(C^{ST}_{ab})$ then $[a,b]$ might contain an $\ashs$ $P'$. 
We temporarily add a variable $Z, D(Z) = [a,b]$ to the set of variables $P^{ST}$.
By the observation above, $P' \cup Z$ is a $\shs$. We simulate the chaining argument, identical to
the procedure outlined in the proof of Theorem~\ref{t:pruning-shs}, inside the interval
$[a,b]$ taking into account the dummy variable $Z$.
During this procedure we consider  $O(n)$ intervals, 
because there are at most $n$ variables inside the intervals. 
As constraints \eqref{eqn::firstOverlappingAlldiff}-\eqref{eqn::pyramidS-S} reached
their fixpoint, simulation takes $O(n)$ time. This allows us to find saturated 
intersection intervals $I_i$ to identify the set  $N(P^S) \cap N(P^T)$.
Then, we prune these values from all variables from $S \cap T \setminus P^{ST}$.
}
There are $O(d^2)$ intervals. Finding $N(P^S) \cap N(P^T)$ takes $O(n + d)$ time 
inside an interval. Enforcing the rule takes $O(nd)$ time. Hence, the total time complexity is $O(nd^3)$.
\end{proof}

From Theorems~\ref{t:pruning-shs} and \ref{t:pruning-ashs} it follows that 

\begin{theorem}
	BC on $\oalldiff$ can be enforced in $O(nd^3)$ time.
\end{theorem}



\begin{myexample}
  \label{e:bc-prop}
We demonstrate the action of
constraints~\eqref{eqn::firstOverlappingAlldiff}-\eqref{eqn::pyramidS-S}.
The interval $[1,4]$ contains a $\shs$ $P = \{X_2,X_3,X_4,X_6\}$.
\textbf{Rule~\ref{t:e:out_inter}.}
$lb(C^S_{12}) \geq 1$ and $lb(C^T_{23}) \geq 1$ 
$\Rightarrow{\eqref{eqn::overlapping-alldiff-S},~\eqref{eqn::lastOverlappingAlldiff}}\Rightarrow$
$ub(C^{ST}_{12}) \leq 1$ and  $ub(C^{ST}_{23}) \leq 1$  
$\Rightarrow{~\eqref{eqn::pyramidST}}\Rightarrow$
$ub(C^{ST}_{13}) \leq 2$. 
The interval $[1,3]$ is saturated, as $lb(C^{ST}_{13}) = ub(C^{ST}_{13})$. Hence, by~\eqref{eqn::firstOverlappingAlldiff}-\eqref{eqn::channelST},
$[1,3]$ is removed from $D(X_5)$. 
\textbf{Rules~\ref{t:e:out_side_one},\ref{t:e:out_side_two}.}
As $lb(C^{ST}_{13}) = 2$ 
$\Rightarrow{~\eqref{eqn::overlapping-alldiff-S}}\Rightarrow$
$ub(C^{S}_{13}) \leq 1$  $\Rightarrow{~\eqref{eqn::pyramidS-S}}\Rightarrow$ $ub(C^{S}_{12}) \leq 1$. 
The interval $[1,2]$ is saturated, as $lb(C^{S}_{12}) = ub(C^{ST}_{12})$.
Hence, by~\eqref{eqn::firstOverlappingAlldiff}-\eqref{eqn::chaneling},\eqref{eqn::channelS}, $[1,2]$ is removed from $D(X_1)$. Similarly, $[2,3]$ is removed from  $D(X_7)$. 
\textbf{Rule~\ref{t:e:out_all}.}  This is satisfied as $2$ is removed from 
all variables outside $P$.
\myOmit{
!!! This part does RC
We continue reasoning and show that the value $2$ is pruned from 
$X_3$ and $X_4$. As $ub(C^{ST}_{12}) \leq 1$ $\Rightarrow$ $lb(C^{ST}_{33}) = 1$ $\Rightarrow$ 
$ub(C^{T}_{33}) = 0$ and $lb(C^{T}_{22}) = ub(C^{T}_{22}) = 1$. Hence, 
$ub(C^{ST}_{22}) = 0$ and  $lb(C^{ST}_{11}) = 1$. Finally,
$ub(C^{S}_{11}) = 0$ and $lb(C^{S}_{22}) = ub(C^{S}_{22}) = 1$.

$$
{\scriptsize
\begin{array}{c c|ccccc} 
& & 1 & 2 & 3 & 4 & 5  \\ \hline
\multirow{2}{*}{$A^S = S \setminus T$}
 & X_1 &  &  & \ast & \ast &   \\ 
 & X_2 &  & \ast & & & \\ 
\hline \hline
\multirow{3}{*}{$A^{ST} = S \cap T$}
 & X_3 &  \ast &  & \ast &  & \\ 
 & X_4 &  \ast &  & \ast &  & \\ 
 & X_5 &   &  &  &  \ast &  \ast\\ 
\hline \hline
\multirow{2}{*}{$A^T = T \setminus S$}
 & X_6 &  & \ast  & & &  \\
 & X_7 &  &   & & \ast& \ast 
\end{array}
}
$$
We can show that the decomposition does the same pruning as
a $DC$ propagator would do on this example.
}
\end{myexample}



\subsection {Exponential separation}

We now give a pathological problem on which
our new propagator does exponentially less
work than existing methods. 

\begin{theorem}
  \label{thm:path-over-alldiff}
  There exists a class of problems 
such that enforcing $BC$ on
  $\oalldiff$ immediately detects unsatisfiability while 
  a search method that enforces DC on the decomposition
 into $\alldiff$ constraints explores an exponential search tree
  regardless of branching. 
\end{theorem}

\begin{proof}

  The instance $\mathcal{I}_n$ is defined as follows
   $ \mathcal{I}_n  =  \alldiff ([X \cup Y]) \land  \alldiff ([Y \cup Z])$,
   $D(X_i) = [1,2n-1]$, $i=1,\ldots,n$, 
   $D(Y_i) = [1,4n-1]$, $i=1,\ldots,2n$ and
   $D(Z_i) = [2n,4n-1]$, $i=1,\ldots,n$.

\myOmit{
  \begin{eqnarray*}
    \mathcal{I}_n & = & \alldiff ([X_1,\ldots,X_{n},Y_{1},\ldots,Y_{2n}]) \land \\
    &&  \alldiff ([Y_{1},\ldots,Y_{2n},Z_{1},\ldots,Z_{n}]) \land \\
    &&  D(X_i) = [1,2n-1], \; \forall i=1,\ldots,n \land \\ 
    &&  D(Y_i) = [1,4n-1], \; \forall i=1,\ldots,2n \land\\
    &&  D(Z_i) = [2n,4n-1], \; \forall i=1,\ldots,n
  \end{eqnarray*}
}
\textbf{$\oalldiff$.} Consider the interval $[1,4n-1]$.
$|P| = 4n$, $|N(P)| = 4n-1$ and $|N(P^S) \cap  N(P^T)| = 0$.
By Theorem~\ref{thm:extend-hall}, 
we detect unsatisfiability. 

\textbf{Decomposition.} 
Consider any $\alldiff$ constraint. A subset of $n$ or fewer variables 
has at least $2n-1$ values in their domains and a subset of $n+1$ to
$3n$ variables has
$4n-1$ values in their domains. Thus, to obtain a Hall set and prune,
we must instantiate at least $n-1$ variables. 
\end{proof}
\myOmit{
It is easy to verify that a pair of $\alldiff$ constraints exhibits
this exponential behaviour. In experiments, even
instances as small as $I_7$ take more than 600 seconds to be solved using
the $DC$ propagator of~\cite{regin1} and are unsolvable when
enforcing $BC$~\cite{lopez1} or $VC$ (pruning equivalent to a clique
of binary inequalities).
}
\myOmit{
We evaluated the performance of $DC$,  $VC$ (clique of binary inequalities) and $BC$~\cite{lopez1} 
 $\alldiff$ propagators  on the example from Theorem~\ref{thm:path-over-alldiff}.
Note that the $\oalldiff$ constraint solves the problem without search while
the time to solve this problem grows exponentially for other propagators.
}

\myOmit{
\begin{table}
\begin{center}
{\scriptsize
\caption{\label{t:t1}  The separation example.
$t$ is time to solve a problem, $\#f$ is the number of backtracks.
Timeout is $1000$ sec.
}
\begin{tabular}{|r|c|c|c|c|} 
\hline 
& {$\oalldiffshort$}
& {$BC$}
& {$VC$}
& {$DC$ } \\ 
\hline 
\hline
 $I_n$& 
 \#f/ 
 t& 
 \#f/ 
 t& 
 \#f/ 
 t& 
 \#f / 
 t \\ 
\hline 
$I_4$   & \textbf{      0} / \textbf{0} & $33\cdot 10^6$/  $513$ & 0 / -  &    $839$ /   0.03 \\ 
 \hline 
$I_5$   & \textbf{      0} / \textbf{0} & 0 / - & 0 / - &  $15119$ /   $1$ \\ 
 \hline 
$I_6$   & \textbf{      0} / \textbf{0} & 0 / - & 0 / - &  $332639$ /  $20$ \\ 
 \hline 
$I_7$   & \textbf{      0} / \textbf{0} & 0 / - & 0 / - & $8648639$ / $666$ \\ 
 \hline 
\end{tabular}}
\end{center}
\end{table}
}

\section{Experimental results}


To evaluate the performance of our decomposition we carried out an
experiment on random problems. We used
Ilog 6.2 on an Intel Xeon 4 CPU, 2.0 GHz, 4GB RAM.  
%
%
We  compare the performance of the $DC$,  $BC$~\cite{lopez1} propagators  
and our decomposition into constraints~\eqref{eqn::firstOverlappingAlldiff}-\eqref{eqn::pyramidS-S} 
for the $\oalldiff$ constraint (\oalldiffshort).
We use randomly generated problems with three 
global constraints:
$\alldiff(X \cup W)$,
$\alldiff(Y \cup W)$
and 
$\alldiff(Z \cup W)$, 
and a linear number of binary 
ordering relations between 
variables in $X$, $Y$ and $Z$. We use a random variable ordering and 
run each instance with 50 different seeds. 
As Table~\ref{t:t1} shows, 
our decomposition 
reduces the search space  significantly, is much faster 
and solves more instances overall.

\begin{table}[htb]
\begin{center}
{\scriptsize
\caption{\label{t:t1}
Random problems.  $n$ is the size of $X$, $Y$ and $Z$, $o$ is the size of $W$,
$d$ is the size of variable domains.
Number of instances solved in 300 sec out of 50 runs / average backtracks/average time to solve.
}
\begin{tabular}{|r|c|c|c|} 
\hline 
 n,d,o 
& {$BC$}
& {$DC$}
&\multicolumn {1}{|c|}{$\oalldiffshort$ } \\ 
\hline 
\hline
 & 
 \#s / 
 \#bt / 
 t& 
 \#s / 
 \#bt / 
 t& 
 \#s / 
 \#bt / 
 t \\  
\hline 
$4,15,10$   & 14 /2429411  /     61.8 & 41 /1491341  /     52.1 & \textbf{42} / \textbf{  17240}  /  \textbf{   32.5} \\ 
 \hline 
$4,16,11$   & 6 /5531047  /    153.7 & 22 /1745160  /     67.9 & \textbf{31} / \textbf{   8421}  /  \textbf{   19.5} \\ 
 \hline 
$4,17,12$   & 1 /     17  /  0 & 6 /2590427  /    100.9 & \textbf{24} / \textbf{   8185}  /  \textbf{   21.5} \\ 
 \hline 
$5,16,10$   & 11 /3052298  /     82.0 & 37 /1434903  /     58.2 & \textbf{42} / \textbf{  20482}  /  \textbf{   48.5} \\ 
 \hline 
$5,17,11$   & 2 /3309113  /     94.5 & 19 /2593819  /    114.6 & \textbf{26} / \textbf{   4374}  /  \textbf{   15.8} \\ 
 \hline 
$5,18,12$   & 1 /     17  /  0 & 4 /2666556  /    133.1 & \textbf{22} / \textbf{   3132}  /  \textbf{   12.2} \\ 
 \hline 
$6,17,10$   & 11 /2845367  /     79.1 & 31 /1431671  /     66.3 & \textbf{40} / \textbf{   6796}  /  \textbf{   21.9} \\ 
 \hline 
$6,18,11$   & 4 / 199357  /      6.6 & 16 /1498128  /     80.2 & \textbf{31} / \textbf{   4494}  /  \textbf{   17.5} \\ 
 \hline 
$6,19,12$   & 4 /3183496  /    110.0 & 5 /1035126  /     66.2 & \textbf{27} / \textbf{   3302}  /  \textbf{   15.5} \\ 
 \hline 
\multicolumn {1}{|r|} { TOTALS }& & &  \\ 
\multicolumn {1}{|r|} {sol/total}& \multicolumn {1}{|c|} { 54 /450}& \multicolumn {1}{|c|} { 181 /450}& \multicolumn {1}{|c|} {\textbf{285} /450}\\ 
\multicolumn {1}{|r|} {avg time for sol}& \multicolumn {1}{|c|}{ 78.072} & \multicolumn {1}{|c|}{ 70.551} & \multicolumn {1}{|c|}{\textbf{ 24.689}} \\ 
\multicolumn {1}{|r|} {avg bt for sol}& \multicolumn {1}{|c|}{2818926} & \multicolumn {1}{|c|}{1666568} &\multicolumn {1}{|c|}{\textbf{   9561}} \\ 
\hline 
\end{tabular}}
\end{center}
\vspace{-2ex}
\end{table}

\myOmit{
\section{Other related work}

There have been many studies on propagation algorithms for a single
\alldiff constraint.  A domain consistency algorithm that runs in
$O(n^{2.5})$ was introduced in~\cite{regin1}. A range consistency
algorithm was then proposed in~\cite{leconte} that runs in time
$O(n^2)$. The focus was moved from range consistency to bound
consistency with~\cite{Puget98}, who proposed a bound consistency
algorithm that runs in $O(n\log n)$. This was later improved further
in~\cite{mtcp02} and then in~\cite{lopez1}.

Little work has been done for propagating conjunctions of \alldiff
constraints. The problem was shown to be NP-hard in~\cite{KEKM08}. In
the same work, the authors proposed approximation algorithms for the
simultaneous matching problem. It is unclear, however, whether they
can be used for building constraint propagation algorithms.

In~\cite{ahhtcp07}, it was proposed that communication between
different constraints can be improved by generalizing pruning to
removing paths from a multi-valued decision diagram. Propagating
conjunctions of \alldiff constraints was proposed as an application,
but the level of consistency enforced in this setting is not
well-defined.
}
\section{Conclusions}

We have generalized Hall's theorem 
to simultaneous matchings in a bipartite graph. This
generalization suggests a polynomial time algorithm to find a
simultaneous matching in a convex bipartite graph. We applied
this to a problem in constraint programming of propagating
conjunctions of \alldifferent constraints. Initial experimental
results suggest that reasoning about such conjunctions can significantly
reduce the size of the explored search space. 
There are several avenues for future research. 
For example, the algorithmic techniques proposed in~\cite{puget98}
and~\cite{lopez1} may be generalizable to simultaneous
bipartite matchings, giving more efficient propagators.
Further, matchings are
used to propagate other constraints such as \nvalue\ \cite{nvalue}. 
It may be possible to apply similar insights 
to develop propagators for conjunctions of
other global constraints, 
or to improve existing propagators for global constraints that
decompose into overlapping constraints
like \sequence\ \cite{sequence}. Finally, we may 
be able to develop polynomial time
propagators for otherwise intractable cases if certain parameters
are fixed \cite{fixed}.

\section*{Acknowledgments}


This research is supported by ANR UNLOC project (ANR 08-BLAN-0289-01),
the Australian 
Government's  Department of Broadband, Communications and the Digital Economy
and the
ARC. 

\bibliographystyle{aaai}



\end{document}